\documentclass[pa]{article}

\usepackage[pdftex]{graphicx}
\usepackage{amssymb}

\newcommand{\argmax}{{\rm argmax}}

\newcommand{\yy}[1]{{\it Draft---{#1}}}
\newcommand{\ct}[1]{{\it Cite---{#1}}}

\newcommand{\trm}[1]{\textit{#1}}
\newcommand{\tv}{{^{\rm TrV}}} 
\newcommand{\sv}{{^{\rm SuV}}} 
\newcommand{\bp}{{^{\rm BiP}}} 
\newtheorem{example}{Example}

\newtheorem{prop}{Proposition}

\newtheorem{theorem}{Theorem}
\newtheorem{corr}{Corollary}
\newenvironment{proof}{\noindent {\bf Proof.}}{  {This concludes the proof.~\hfill~$\blacksquare$}}

\newcommand{\myhead}{Biased Communications to Trusting and Suspicious Voters}
\title{The Effect of Biased Communications \\ On Both Trusting and
  Suspicious Voters}

\author{William W. Cohen \\
  Carnegie Mellon University \\
  Department of Machine Learning
  \and David P. Redlawsk\\
  Rutgers University \\
  Department of Political Science
  \and Douglas Pierce\\
  Rutgers University \\
  Department of Political Science
}

\begin{document}

\maketitle

\abstract{In recent studies of political decision-making, apparently
  anomalous behavior has been observed on the part of voters, in which
  negative information about a candidate strengthens, rather than
  weakens, a prior positive opinion about the candidate.  This
  behavior appears to run counter to rational models of decision
  making, and it is sometimes interpreted as evidence of non-rational
  ``motivated reasoning''. We consider scenarios in which this effect
  arises in a model of rational decision making which includes the
  possibility of deceptive information.  In particular, we will
  consider a model in which there are two classes of voters, which we
  will call \trm{trusting voters} and \trm{suspicious voters}, and two
  types of information sources, which we will call \trm{unbiased
    sources} and \trm{biased sources}.  In our model, new data about a
  candidate can be efficiently incorporated by a trusting voter, and
  anomalous updates are impossible; however, anomalous updates can be
  made by suspicious voters, if the information source mistakenly
  plans for an audience of trusting voters, and if the partisan goals
  of the information source are known by the suspicious voter to be
  ``opposite'' to his own.  Our model is based on a formalism
  introduced by the artificial intelligence community called
  ``multi-agent influence diagrams'', which generalize Bayesian
  networks to settings involving multiple agents with distinct goals.
}

\section{Introduction}

Historically, political decision-making has been modeled in a number
of ways.  Models that propose rational decision making on the part of
voters must account for the fact that voters frequently have
difficulty in responding to factual surveys on political issues.  One
resolution to this difficulty is to model candidate evaluation as an
online learning process, in which a tally representing candidate
affect is incremented in response to external information
\cite{Lodge1989}, after which the information itself is discarded.
However, in a number of recent studies of political decision-making,
apparently anomolous behavior has been observed on the part of voters,
in which \trm{negative} information about a candidate $k$ strengthens
(rather than weakens) a prior positive opinion held about $k$
\cite{civettini_voters_2009,redlawsk_hot_2002} . 

This behavior appears to run counter to rational models of decision
making, and it is sometimes interpreted as evidence of non-rational
``motivated reasoning'' \cite{Kunda_1990}.  In motivated reasoning
models, a voter will (1) evaluate the affect of new information—i.e.,
its positive or negative emotional charge, then (2) compare this to
the affect predicted by current beliefs, and finally (3) react, where
\trm{congruent} information (i.e., information consistent with
predicted affect) is processed quickly and easily, and incongruent
information is processed by a slower “stop-and-think”
process. “Stop-and-think” processing may include steps such as
counter-arguing, discounting the validity of the information, or
bolstering existing affect by recalling previously-assimilated
information \cite{Lodge_Taber_2000,redlawsk_hot_2002}.

Some evidence for the motivated reasoning hypothesis comes from
hman-subject experiments using a \trm{dynamic process tracing
  environment} (DPTE), in which data relevant to a mock election is
presented as a dynamic stream of possibly relevant news items.  In
DPTE experiments, detailed hypotheses about political reasoning can be
tested, for instance by varying the frequency and amount of
incongruent information presented to voters in the mock election.
Experimental evidence shows, for instance, that both political
sophisticates and novices spend more time processing negative
information about a liked candidate, and novices also spend longer
processing positive information about a disliked candidate
\cite{redlawsk_hot_2002}. Most intriguingly, small to moderate amounts
of incongruent information---e.g., negative information about a liked
candidate---actually reinforce the prior positive view of the
candidate \cite{redlawsk_tipping_2010}.

This apparently anomalous effect---whereby information has the inverse
of the expected impact on a voter---appears to be inconsistent with
rational decision-making.  In this paper, we show analytically that
this ``anomalous'' effect can occur in a model of rational decision
making which includes the possibility of deceptive information.  The
model makes another interesting prediction: it justifies as
computationally effective and efficient a heuristic of pretending to
believe information from a possibly-deceptive source if that source's
political preferences are the same as the voters.

In particular, we will consider a model in which there are two classes
of voters, which we will call \trm{trusting voters} and
\trm{suspicious voters}, and two types of information sources, which
we will call \trm{unbiased sources} and \trm{biased sources}.
Information from an unbiased source is modeled simply as data $D_k$
that probabilistically inform a voter about a candidate $k$'s
positions.  \trm{Trusting voters} are voters that treat information
about a candidate as coming from an unbiased source.  We show that, in
our model, new data about a candidate can be efficiently incorporated
by a trusting voter, and anomalous updates (in which ``negative''
information increases support) are impossible.

\trm{Biased sources} are information sources $j$ who plan their
communications with a goal in mind (namely, encouraging trusting
voters to vote in a particular way).  To do this, $j$ will access some
data $C_{k}$ which is communicated to them only (not to voters
directly) and release some possibly-modified version $B_{k}$ of
$C_{k}$, concealing the original $C_{k}$.  $B_{k}$ is chosen based on
the utility to $j$ of the probable effect of $B_{k}$ on a trusting
voter $i$.

We then introduce \trm{suspicious voters}.  Unlike trusting voters,
who behave as if communications were from unbiased sources, suspicious
voters explicitly model the goal-directed behavior of biased sources.

The behavior of rational suspicious voters depends on
circumstances---depending on the assumptions made, different effects
are possible.  If the partisan goals of the biased source $j$ and a
suspicious voter $i$ are aligned, then a suspicious voter can safely
act as if the information is correct---i.e., perform the same updates
as a naive voter.  Intuitively, this is because $j$ is choosing
information $B_k$ strategically to influence a naive voter to achieve
$j$'s partisan goals, and since $i$'s goals are the same, it is
strategically useful for $i$ to ``play along'' with the deception;
this intutition can be supported rigorously in our model.  If the
partisan goals of $j$ are unknown, then a rational suspicious voter
$i$ may discount or ignore the information $B_k$; again, this
intuition can be made rigorous, if appropriate assumptions are made.
Finally, if the partisan goals of $j$ are known to be ``opposite''
those of $i$, then a rational suspicious voter may display the
``anomalous'' behavior discussed above: information $B_k$ that would
cause decreased support for a suspicious voter will cause increased
support for $i$.  Intuitively, this occurs because $i$ recognizes that
$j$ may be attempting to decrease support for candidate $k$, and since
$i$ and $j$ have ``opposite'' partisan alignments, it is rational for
$i$ to instead increase support.

In short, in this model, a negative communication about $k$ can have
the effect opposite to one's initial expectation; however, the
apparent paradox is not due to motivated reasoning, but simply to
imprecise planning on the part of $j$.  In particular, $j$'s
communication was planned by a biased source with the aim of
influencing trusting voters, while in fact, $i$ is a suspicious voter.

Below, we will first summarize related work, and then flesh out these
ideas more formally.  Our model will be based on a formalism
introduced by the artificial intelligence community called
``multi-agent influence diagrams'', which generalize Bayesian networks
to settings involving multiple agents with distinct goals.

\section{Related work} \label{sec:related}

This work is inspired for recent work on motivated reasoning and hot
cognition in political contexts (for recent overviews, see
\cite{graber_political_2005,redlawsk_feeling_2006}).  There is strong
experimental evidence that information processing of political
information involves emotion, and there recent research has sought to
either collect empirical evidence for
\cite{redlawsk_hot_2002,civettini_voters_2009}, and and build models
that explicate \cite{lodge_first_2006}, the mechanisms behind this
phenomenon.

The models of this paper are not intended to dispute role of emotion
in political decision making.  Indeed, our models reflect situations
in which one party deliberately withholds or distorts information to
manipulate a second party, and introspection clearly suggests that
such situations will typically invoke an emotional response.  However,
work in social learning (e.g., \cite{rendell_why_2010}) and
information cascades (e.g., \cite{salganik_experimental_2006}) shows
that behaviors (such as ``following the herd'') which appear to be
driven by non-rational emotions may in fact be strategies that are
lead to results that are evolutionarily desirable (if not always
``rational'' from the individual's perspective.)  Hence, the
identification of emotional aspects to decision-making does not
preclude rational-agent explanations; rather, it raises the question
of \emph{why} these mechanisms exist, what evolutionary pressures
might cause them to arise, and whether or not those pressures are
still relevant in modern settings.  This paper makes a step toward
these long-term goals by identifying cases in which behavior
explainable by motivated reasoning models is also rational, for
instance in the result of Theorem~\ref{thm:motivated}.

The explanation suggested here for anomalous, motivated-reasoning-like
updating is based on a voter recognizing that a source may be biased,
and correcting for that bias.  While this explanation of anomalous
updates is (to our knowledge) novel, it is certainly recognized that
\emph{trust} in the source of information is essential in political
persuasion, and that a voter's social connextions strongly influences
political decision-making (e.g., \cite{beck_social_2002}).  More
generally, empirical studies of persuasion substantiate a role of
confirmatory bias and prior beliefs \cite{dellavigna_persuasion_2010}, and show that in
non-political contexts (e.g., in investing), sophisticated consumers
of information adjust for credibility, while inexperienced agents
under-adjust.

The results of this paper are also related to models of media
choice---for instance, research in which the implications of a
presumed tendency of voters to seek confirmatory news is explored
mathematically
\cite{burke_introductions_2007,bernhardt_political_2008,duggan_spatial_2011,stone_ideological_2011}.
Other analyses show why preferences for unbiased news lead to economic
incentives to distort the news \cite{burke_unfairly_2007}.  This paper
does not address these issues, but does contribute by providing a
rational-agent model for why such a confirmatory bias exists: in
particular, Proposition~\ref{prop:same-alignment-utility} describes
a strategy from using information from biased information sources with
similar preferences as a voter. We notice that this strategy is both
simple and computationally inexpensive, and might be preferable on
these grounds to more complex strategies to ``de-noise'' biased
information from sources with unknown preferences.

A further connection is to formal work on ``talk games'', such as
Crawford and Sobel's model of strategic communication
\cite{crawford_strategic_1982}.  In this model, a well-informed but
possibly deceptive ``sender'' sends a message to a ``receiver'' who
(like our voter) takes an action that affects both herself and the
sender.  Variants of this model have explored cases in which
information can only be withheld or disclosed, and disclosed
information may be verified by the receiver
\cite{milgrom_relying_1986}; cases where the receiver uses approximate
``coarse'' reasoning \cite{mullainathan_coarse_2008}; and cases where
there is a mixed population of strategic and naive recievers, all of
whom obtain information from senders acting strategically
\cite{ottaviani_naive_2006}.

The analysis goals in ``talk games'' is different from the goals of
this paper.  whereas we investigate whether specific counter-intuitive
observed behavior can arise in a plausible (not necessarily temporally
stable) situation, this prior work primarily nalyzes the communication
efficiency of a a system {in equilibrium}, Some of the results
obtained for talk games are reminiscent of results shown here: for
instance, Crawford and Sobel show that in equilibrium, signaling is
more informative when the sender and reciever's preferences are more
similar.  However, other results are less intuitive: for instance, in
some models there is no deception at equilibrium
\cite{ottaviani_naive_2006}.  We note that while equilibria are
convenient to analyze, there is no particular reason to believe that
natural political discourse reaches an equilibria.

Game theory has a long history in analyses of politics; in particular,
writing in 1973, Shubik discusses possible applications of game theory
to analysis of misinformation \cite{shubik_game_1973}.  The tools used
used in this paper arose from more recent work in artificial
intelligence \cite{koller_multi-agent_2003,pfeffer_networks_2008},
specifically analysis of multi-agent problem solving tasks, in which
one agent explicitly models the goals and knowledge of another in
settings involving probabilistic knowledge.  One small contribution of
this paper is introduction of a new set of mathematical techniques,
which (to our knowledge) have not been previously used for analysis of
political decision-making.  We note however that while these tools are
convenient, they are not absolutely necessary to obtain our results.

\section{Modeling trusting voters}

\subsection{A model}

\begin{table}

\yy{to check: subscripts for $C,B,R$}

\begin{tabular}{ll}
\hline
$i$		& a voter \\
$j$		& a pundit \\
$k$             & a candidate \\             
$T_i,T_j$       & ``target positions'' for voter $i$ and source $j$ \\
$T_k$           & the position of a candidate $k$\\
$dom(T)$        & the set of values taken on by random variable $T$\\
$S_{ik},S_{jk}$  & similarity of a target position and a candidate \\
$D_{k}$	        & data about candidate $k$\\
$Y_{ik}$ 	& the vote of $i$ for $k$\\
$U_i$ 	        & utility function for $i$\\
\hline
$C_{jk}$	& data about $k$ that is known only to $j$\\
$B_{jk}$	& biased variant of data $C_{jk}$ about $k$ that has been modified by $j$\\
$R_{jk}$	& reputational cost to pundit $j$ of modifying $C_{jk}$ to $B_{jk}$\\
\hline
$P\tv(X|Y)$     & a conditional probability computed using the trusting voter model\\
$P\bp(X|Y)$     & a conditional probability computed using the biased pundit model\\
$P\sv(X|Y)$     & a conditional probability computed using the suspicious voter model\\
\hline
\end{tabular}
\caption{Notation used in the Paper}
\label{tab:notation}
\end{table}

\begin{figure}
\centerline{\includegraphics[width=0.7\textwidth]{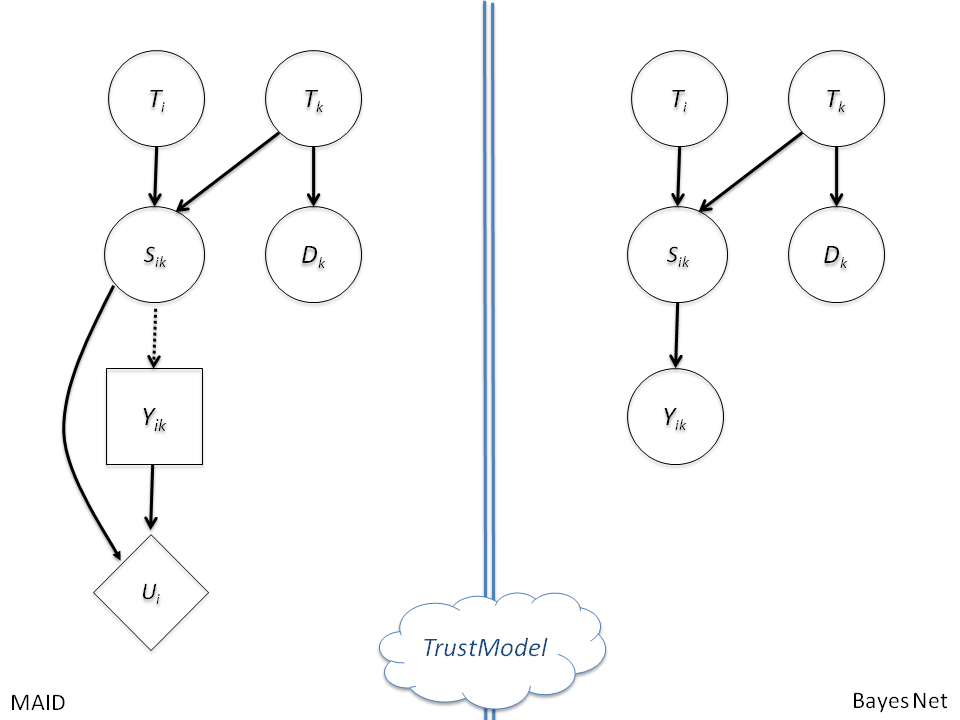}}
\caption{Model of a trusting voter, in MAID and Bayes Net notation}
\label{fig:trusting-voter-model}
\end{figure}

Consider the influence diagram on the left-hand side of
Figure~\ref{fig:trusting-voter-model}.  In this model $i$ is a voter,
and $k$ is a candidate.  A voter $i$ has a preference $T_i$: for
example, $T_i$ might be a member of the set
\[ dom(T) = \{ {\rm goodLiberal}, {\rm evilLiberal}, {\rm goodConservative}, {\rm evilConservative} \}
\] 
Likewise $T_k$ is candidate $k$'s actual position, which is also an
element of $dom(T)$.  $S_{ik}$ measures how similar two positions are.
$Y_{ik}$ is a measure of $i$'s support for $k$, which is chosen by
voter $i$ to maximize $i$'s utility.  The utility $U_i$ for voter $i$
is a function of how appropriate $Y_{ik}$ is given $S_{ik}$.  Finally,
$D_{k}$ is some data about $k$, generated probabilistically according
to the value of $T_k$.  As an example, this might be a statement made
by $k$.  The notation used in this diagram (and elsewhere below) is
summarized in Table~\ref{tab:notation}.

The shapes in the nodes in the diagrams indicate the type of the
variable: diamonds for utilities, circles for random variables, and
squares for \textit{decision variable}, which an agent (in this case,
voter $i$) will choose in order to maximize utility.  The dotted arrow
lines leading into a decision variable indicates information available
when the decision is made.  The arrows leading into a random variable
node indicate ``parents''---variables on which the value of the
variable is conditioned. This sort of diagram is called an influence
diagram, and the general version we will use later (in which multiple
agents may exist) is called a MAID (Multi-Agent Influence Diagram)
\cite{pfeffer_networks_2008,koller_multi-agent_2003}.

More precisely the model defines a probability distribution
generated by this process:
\begin{itemize}
\item Pick $t_i \sim P(T_i)$, where $P(T_i)$ is a prior on voter preferences.
\item Pick $t_k \sim P(T_k)$, where $P(T_k)$ is a prior on candidate positions.
\item Pick $d_k \sim P(D_k|T_k=t_k)$. Equivalently, we could let $d_k=f_D(t_k,\epsilon_D)$, where $f_D$ is a 
function and $\epsilon_D$ is a random variable chosen independently of all other random variables in the model.
\item Pick $s_{ik} \sim P(S_{kij}|T_k=t_k,T_j=t_j)$. Equivalently, we could let $s_{ik}=f_S(t_i,t_k,\epsilon_S)$, where $f_S$ is a 
function and $\epsilon_S$ is a random variable, again chosen independently of all other random variables in the model.
\item Allow voter $i$ to pick $y_{ik}$, based on a user-chosen probability distribution
  $P_\tau(Y_{ik}|S_{ik}=s_{ik})$, or equivalently computed using $f_Y(s_{ik},\epsilon_Y)$.
\item Pick utility $u_i$ from $P(U|y_{ik},s_{ik})$---or
  equivalently, computed as $u_i = f_U(y_{ik},s_{ik},\epsilon_U)$.
\end{itemize}
The user can choose any distribution $P_\tau(Y|S)$, but we will
henceforth assume that she will make the optimal choice---i.e., the
probability distribution $P_\tau$ will be chosen by $i$ to maximize
the expected utility $u_i$, where $u_i$ will be picked from
$P(U|y_{ik},S_{ik})Pr(S_{ik})$---or equivalently computed as
$u_i=\sum_{s'} P(S_{ik}=s') * f_U(y_{ik},s',\epsilon_U)$.  In this
specific case, the conversion is based on the observation that
\begin{equation} \label{eq:trust}
 P\tv(Y_{ik}=y|S_{ik}=s) = P\left(y=\argmax_{y'} \int_{\epsilon_U} f_u(y',s,\epsilon_U) d\epsilon_U\right)
\end{equation}
and yields the Bayes network on the righthand side of
Figure~\ref{fig:trusting-voter-model}. Here $Y_{ik}$ is simply a
random variable conditioned on $S_{ik}$, where the form of the
dependency depends on Equation~\ref{eq:trust}.  Notice that the link
from $S_{ik}$ to $Y_{ik}$ is deterministic.  We call this the
\trm{trusting voter model}, since voter $i$ trusts the validity of the
information $D_k$.

MAIDs (and their single-agent variants, influence diagram networks)
have a number of advantages as a formalism.  They provide a compact
and natural computational representation for situations which are
otherwise complex to describe - in particular, situations in which
agents have limited knowledge of the game structure, or mutually
inconsistent beliefs, but act rationally in accordance with these
beliefs.  In particular, MAIDs relax the assumption usually made in
Bayesian games that players' beliefs are consistent, and supports an
explicit process in which one player can model another player's
strategy.  MAIDs also support an expressive structured representation
for a player's beliefs.  Together these features make them appropriate
for modeling ``bounded rationality'' situations of the sort we
consider here.  Further discussion of MAIDs, and their formal relation
to other formalisms for games and probability distributiobns, can be
found elsewhere \cite{koller_multi-agent_2003,pfeffer_networks_2008}.

Next, we will explore some simplifications of Eq.~\ref{eq:trust}) If
$f_u$ is deterministic, then Eq.~\ref{eq:trust} simplifies to the
following choice of $y_{ik}$ given $s_{ik}$:
\[
y_{ik} = f_y\tv(s_{ik}) = \argmax_{y'} f_u(y',s_{ik})
\]
If only a distribution $P(S_{ik}=s)$ is known, then voter $i$'s
optimal strategy is to let
\[
y_{ik} = \argmax_{y'} \sum_s f_u(y',s) P(S_{ik}=s) 
\]

In any case, however, the model has similar properties: once we assume
that $i$ uses an optimal strategy, then the MAID becomes an ordinary
Bayes net, defining a joint distribution over the variables $T_i$,
$T_k$, $S_{ik}$, and $Y_{ik}$.  We will henceforth use $P\tv$ to
denote probabilities computed in this model, and reserve the
non-superscripted $P(A|B)$ and $P(A)$ for conditional (respectively
prior) probability distribution that are assumed to be available as
background information.

\begin{table}[tb]
\begin{small}
\begin{tabular}[t]{|l|r|}
\hline
$d$  & $P(T_i=d)$ \\
\hline
goodLiberal & 0.4 \\
goodConserv & 0.4 \\
evilLiberal & 0.1 \\
evilConserv & 0.1 \\
\hline
\multicolumn{2}{c}{~} \\
\hline
$d$  & $P(T_k=d)$ \\
\hline
goodLiberal & 0.29 \\
goodConserv & 0.69 \\
evilLiberal & 0.01 \\
evilConserv & 0.01 \\
\hline
\end{tabular}~\begin{tabular}[t]{|ll|r|}
\hline
$t_i$            & $t_k$       & $s_{ik}|t_i,t_k$ \\
\hline
goodLiberal      & goodLiberal &  $5 + \epsilon$ \\
goodLiberal      & goodConserv &  $1 + \epsilon$ \\
goodLiberal      & evilLiberal & $-2 + \epsilon$ \\
goodLiberal      & evilConserv & $-5 + \epsilon$ \\
goodConserv      & goodConserv &  $5 + \epsilon$ \\
goodConserv      & evilLiberal & $-5 + \epsilon$ \\
goodConserv      & evilConserv & $-2 + \epsilon$ \\
evilLiberal      & evilLiberal &  $5 + \epsilon$ \\
evilLiberal      & evilConserv & $-5 + \epsilon$ \\
evilConserv      & evilConserv &  $5 + \epsilon$ \\
\hline
\end{tabular}~~\begin{tabular}[t]{|ll|l|}
\hline
$t_k$ & $c_k$ & $P(c_k|t_k)$ \\
\hline
goodLiberal & safety-net & 0.4 \\
goodLiberal & motherhood & 0.6 \\
goodConserv & guns       & 0.3 \\
goodConserv & motherhood & 0.7 \\
evilLiberal & safety-net & 0.9 \\
evilLiberal & chthulu    & 0.1 \\
evilConserv & guns       & 0.8 \\
evilConserv & chthulu    & 0.2 \\ 
\hline
\multicolumn{3}{c}{ ~ } \\
\multicolumn{3}{c}{ $P(Y_i=1|S_{ik}=s)$: see Eq~\ref{eq:trust} }
\end{tabular}
\end{small}
\caption{Conditional probability tables (CPT) for a sample example of
  the trusting voter model.  In the table for $s_{ik}|t_i,t_k$,
  $\epsilon$ is drawn unformly from the set $\{-1,0,+1\}$ and for
  pairs $t_i,t_k$ not shown, $s_{ik}|t_i,t_k=s_{ik}|t_k,t_i$.}
\label{tab:cpt}
\end{table}

\begin{example}
To take a concrete example, igure~\ref{tab:cpt} shows the conditional
probability tables for a small example, where candidates and target
positions have values like \textit{evilLiberal} and
\textit{goodConserv}, and the value of a communication $C_k$ is a name
of something that a candidates might support (e.g.,\textit{motherhood}
or \textit{guns}).
\end{example}

\subsection{Implications of the model}

\subsubsection{New information about a candidate is easy to process}
\label{sec:trusting-inference}

It is obvious how to use the trusting voter model to compute $Y_{ik}$
if $T_k$ and $T_i$ are known.  However, a more reasonable situation is
that $D_k$ and $T_i$ are known to $i$, but the true position of the
candidate can only be inferred, indirectly, from $D_k$.  Fortunately,
using standard Bayes network computations, we can also easily compute
a distribution over $Y_{ik}$ given $D_k$.

First, note that we can marginalize over $T_k$ and compute
\[ P(D_k=d) = {\sum_{t'}P(D_k=d|t')P(T_k=t')}
\]
and that using Bayes' rule 
\begin{eqnarray}  \label{eq:tk-from-d}
P(T_k=t|D_k=d)  =  \frac{P(D_k=d|T_k=t)P(T_k=t)}{P(D_k=d)} = \frac{P(D_k=d|T_k=t)}{P(D_k=d)} P(T_k=t)
\end{eqnarray} 
This gives a simple rule for $i$ to use in choosing her vote $Y_{ik}$ given $D_k$:
\begin{eqnarray} 
\lefteqn{P\tv(Y=y|T_i=t_i,D_k=d_k)} & & \label{eq:y-from-tk}\\
   & = & \sum_{s}P\tv(Y=y|S_{ik}=s)\sum_{t}P(S_{ik}=s|T_i=t_i,T_k=t)P(T_k=t|D_k=d_k) \\
   & = & P\tv(Y=y|S_{ik}) P(S_{ik}|t_i,T_k)P(T_k|d_k) \nonumber
\end{eqnarray} 
The last line uses a simplified notation from the Bayes net community,
where sums used to marginalize are omitted, and the event $X=x$ is
replaced with $x$ when the variable $X$ is clear from context.

\begin{figure}
\centerline{\includegraphics[width=0.7\textwidth]{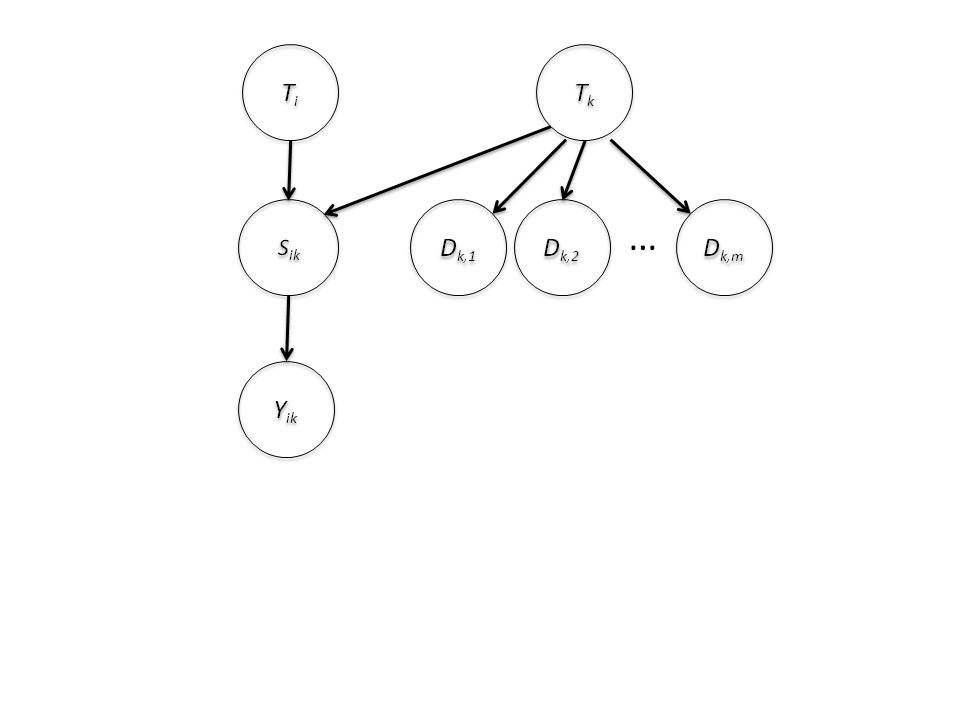}}
\caption{Model of a trusting voter with $m$ multiple independent data items about the candidate, in Bayes Net notation}
\label{fig:mds-trusting-voter-model}
\end{figure}

This procedure can be extended easily to the case of multiple
data items $D_{k,1},\ldots,D_{k,m}$ about the candidate, each
independently generated from $P(D_k|T_k)$, as shown in
Figure~\ref{fig:mds-trusting-voter-model}.  It can be shown that
\[
P(T_k=t|d_{k,1},\ldots,d_{k,m})  =  
    \prod_{\ell}\frac{P(d_{k,\ell}|T_k=t)}{P(d_{k,\ell)}} P(T_k=t)
\]
Put another way, we can define $P(T_k=t|d_{k,1},\ldots,d_{k,m})$ recursively as follows:
\[
P(T_k=t|d_{k,1},\ldots,d_{k,m})  =  
     \frac{P(d_{k,m}|T_k=t)}{P\tv(d_{k,m)}} \cdot P(T_k=t|d_{k,1},\ldots,d_{k,m-1})
\]
Hence, voter $i$ can quickly update beliefs about $T_k$
\textit{incrementally} with each new piece of information $d_{k,\ell}$
and then (as before) use Equation~\ref{eq:y-from-tk} to update her
vote.

This incremental update property is worth emphasizing---while it may
be complicated (if not computationally complex) to compute
$P\tv(Y|S)$, or it may be difficult for a voter to establish her
preferences $t_i$, absorbing new information in the trusting voter
model is straightforward, and consists of two ``natural'' steps:
estimating the candidate's position $T_k$ given the information
$d_{k,\ell}$; and then updating her support $y_{ik}$ for candidate $k$,
given the updated estimate of the candidate's position.

\subsubsection{Positive information increases support}

In a number of recent studies of political decision-making,
\trm{negative} information about a candidate $k$ has been observed to
strengthen (rather than weaken) a voter's support for $k$.  We will
show that under fairly reasonable assumptions this effect can {\em
  not} occur with the trusting voter model.  In particular, we will
assume that support $Y_{ik}$ increases monotonically with the
similarity $S_{ik}$ of the candidate's position $T_k$ and the voter's
preference $T_i$.  

Recall that in the trusting voter model $Y_{ik}$ is a deterministic
function of $S_{ik}$, defined as
\begin{equation} 
f_Y\tv(s_{ik}) = \argmax_{y'} \int_{\epsilon_U} f_u(y',s,\epsilon_U) d\epsilon_U
\end{equation}
If $f_Y\tv$ has the property that 
\[ \forall s_1>s_2, f_Y\tv(s_1)\geq{}f_y\tv(s_2)
\]
then we will say that \textit{$i$'s support for $k$ increases
  monotonically with $s_{ik}$}.  This is one assumption needed for our
result.

We also need to precisely define ``negative'' information.  We say
that $d_k$ is \textit{strictly negative about $k$ to voter $i$} if
there is some partition of $dom(S_{ik})$ into triples
$(a_1,b_1,\delta_1)$, \ldots,$(a_m,b_m,\delta_m)$ so that
\begin{itemize}
\item For every triple $(a_\ell,b_\ell,\delta_\ell)$, $a_\ell<b_\ell$,
  $\delta_\ell>0$, $P(S_{ik}=a_\ell|d_k) = P(S_{ik}=a_\ell) +
  \delta_\ell$, and $P(S_{ik}=b_\ell|d_k) = P(S_{ik}=b_\ell) -
  \delta_\ell$.  In other words, learning $D_k=d_k$ shifts some
  positive probability mass $\delta_\ell$ from the larger similarity
  value $b_\ell$ to the smaller similarity value $a_\ell$.
\item For all $s\in dom(S_{ik})$ that are not in any triple,
  $P(s_{ik}=s|d_k)=P(s_{ik}=s)$.  In other words, the probability mass
  of values $s$ not in any triples is unchanged.
\end{itemize}

\begin{figure}
\centerline{\includegraphics[width=0.7\textwidth]{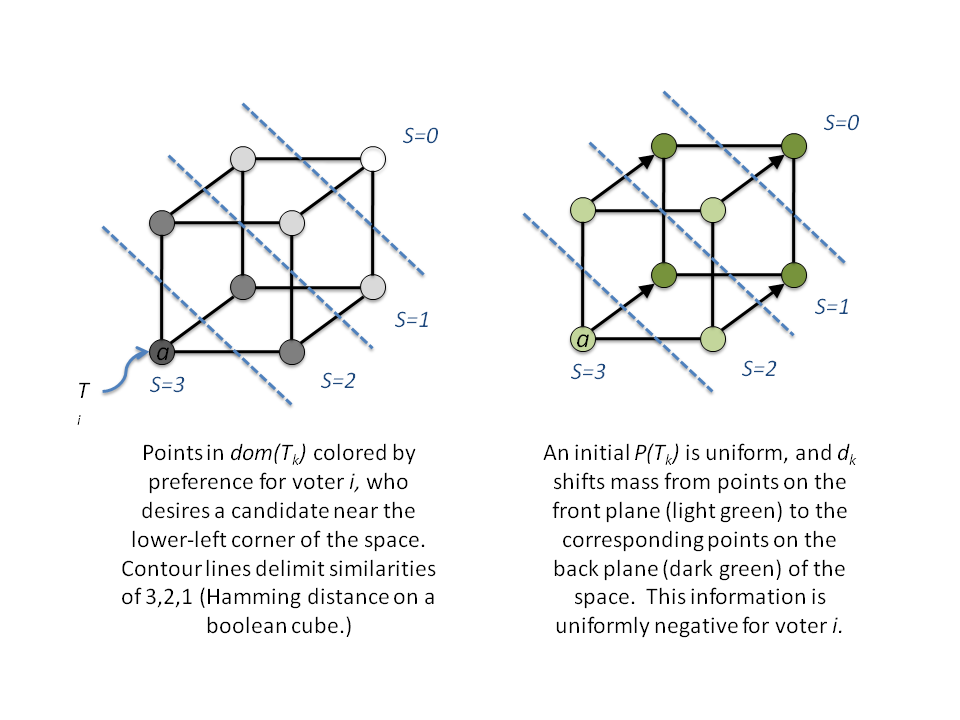}}
\caption{An illustration of strictly negative information for a voter $i$}
\label{fig:neg-info}
\end{figure}

To illustrate this, consider Figure~\ref{fig:neg-info}, which
illustrates a plausible example of strictly negative information.
Strictly positive information is defined analogously.

We can now state formally the claim that negative information will
reduce support.  An analogous statement holds for strictly positive
information.

\begin{theorem} \label{thm:strictly-neg}
Let $E_P[X]$ denote the expected value of $X$ in probability
distribution $P$.  In a trusted voter model $P\tv$, if voter $i$'s
support for $k$ increases monotonically with $s_{ik}$ and $d_k$ is
strictly negative about $k$ to voter $i$, then
$E_{P\tv}[Y_{ik}|D_k=d_k] < E_{P\tv}[Y_{ik}]$.
\end{theorem}

\begin{proof}
Let $S'=dom(S_{ik})-\{a_1,b_1,\ldots,a_m,b_m\}$, where the $a_\ell,b_\ell$'s are
the triples guaranteed by the strictly-negative property of $d_k$.
\begin{eqnarray*}
\lefteqn{E_{P\tv}[Y_{ik}|D_k=d_k]} && \\
   & = & \sum_y y \cdot P\tv(Y_{ik}=y|D_k=d_k) \\
   & = & \sum_s f_y\tv(s)P\tv(S_{ik}=s|D_k=d_k) \\
   & = & \sum_{s'\in{}S'} f_y\tv(s)P\tv(S_{ik}=s|d_k) 
       + \sum_{\ell=1}^m f_y\tv(a_\ell)P\tv(S_{ik}=a_\ell|d_k)  
       + \sum_{\ell=1}^m f_y\tv(b_\ell)P\tv(S_{ik}=b_\ell|d_k)  
\end{eqnarray*}
Looking at the three terms of the final sum in turn, clearly
\[ \sum_{s'\in{}S'} f_y\tv(s)P\tv(S_{ik}=s|d_k) = \sum_{s'\in{}S'} f_y\tv(s)P\tv(S_{ik}=s) 
\]
and the last two terms can be written as
\begin{eqnarray*}
 &    & \sum_{\ell=1}^m f_y\tv(a_\ell)P\tv(S_{ik}=a_\ell|d_k) + f_y\tv(b_\ell)P\tv(S_{ik}=b_\ell|d_k) \\
 &=   & \sum_{\ell=1}^m f_y\tv(a_\ell)\left(P\tv(S_{ik}=a_\ell) + \delta_\ell \right) + f_y\tv(b_\ell)\left( P\tv(S_{ik}=b_\ell) - \delta_\ell \right) \\
 &=   & \sum_{\ell=1}^m f_y\tv(a_\ell) P\tv(S_{ik}=a_\ell) + f_y\tv(b_\ell) P\tv(S_{ik}=b_\ell)
                 + \delta_\ell (  f_y\tv(a_\ell) -  f_y\tv(b_\ell) ) \\ 
 &\leq& \sum_{\ell=1}^m f_y\tv(a_\ell) P\tv(S_{ik}=a_\ell) + f_y\tv(b_\ell) P\tv(S_{ik}=b_\ell)
\end{eqnarray*}
with the last step holding because (a) $\delta_\ell>0$ and (b) $a_\ell<b_\ell$.  (Recall that 
from the monotonicity of $f_y$, if $a_\ell<b_\ell$ then $f_y\tv(a_\ell) \leq{}  f_y\tv(b_\ell)$).
Combining these gives that 
\begin{eqnarray*}
\lefteqn{E_{P\tv}[Y_{ik}|D_k=d_k]} && \\
   & = & \sum_{s'\in{}S'} f_y\tv(s)P\tv(S_{ik}=s|d_k) 
       + \sum_{\ell=1}^m f_y\tv(a_\ell)P\tv(S_{ik}=a_\ell|d_k)  
       + \sum_{\ell=1}^m f_y\tv(b_\ell)P\tv(S_{ik}=b_\ell|d_k)  \\
  &\leq& \sum_{s'\in{}S'} f_y\tv(s)P\tv(S_{ik}=s) + \sum_{\ell=1}^m f_y\tv(a_\ell) P\tv(S_{ik}=a_\ell) + f_y\tv(b_\ell) P\tv(S_{ik}=b_\ell) \\
  & =  & \sum_{s\in{}dom(S_{ik})} f_y\tv(s)P\tv(S_{ik}=s) = E_{P\tv}[Y_{ik}]
\end{eqnarray*}
\end{proof}

\section{Modeling biased pundits and suspicious voters}

\subsection{Biased pundits}

\begin{figure}
\centerline{\includegraphics[width=0.7\textwidth]{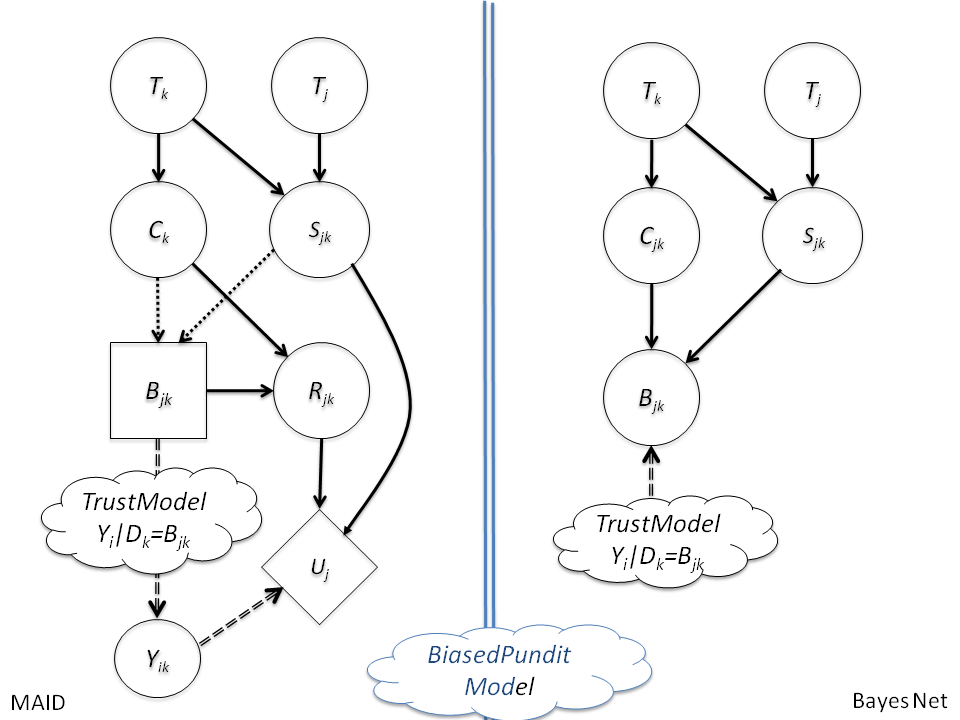}}
\caption{Model of a biased pundit, in MAID and Bayes net notation}
\label{fig:biased-pundit-model}
\end{figure}

We now consider a new model, as shown in
Fig~\ref{fig:biased-pundit-model}.  In this model there is a pundit
$j$, who, like a voter, has a target candidate position $T_j$.  Pundit
$j$ observes a private datapoint $C_k$ from candidate $k$---perhaps
based on a private communication or research---and then publishes a
``biased version'' $B_k$ of $C_k$.  However, $B_k$ is chosen under the
assumption that some trusting voter $i$ will react to $B_k$ {\em as if
  it were $D_k$ in the trusting voter model of
  Figure~\ref{fig:trusting-voter-model}}.  Specifically, we assume
that $B_k$ will provoke a vote $Y_{ik}$ according to the trusting
voter model.  The utility assigned by $j$ to this outcome is a
function of $i$'s vote $Y_{ik}$, the similarity of $S_{jk}$ of $T_k$
to $j$'s target $T_j$, and a ``reputational cost'' $R_{jk}$, which is
a function of $C_k$ and $B_k$. For instance, $R_{jk}$ might be zero if
$b_k=c_k$, and otherwise some measure of how embarrassing it might be
to $j$ if his deception of replacing $c_k$ with $b_k$ were
discovered.

More precisely the model defines a probability distribution
generated by this process:
\begin{itemize}
\item Pick $t_j \sim P(T_j)$, where $P(T_j)$ is a prior on pundit preferences.
\item Pick $t_k \sim P(T_k)$, where $P(T_k)$ is a prior on candidate positions.
\item Pick $c_k \sim P(D_k|T_k=t_k)$, or equivalently, $c_k=f_D(t_k,\epsilon_D)$.  (Notice that we assume 
$c_k$ is chosen from the same conditional distribution $P(D_k|T_k)$ used in the trusting voter model
to chose $D_k$---we're using a different variable here to emphasize the different role it will play.)
\item Allow pundit $j$ to pick $b_{k}$, based some user-chosen distribution $P_\sigma(B_k|C_k=c_k)$.
\item Pick $r_{ik} \sim P(R_{jk}|B_k=b_k,C_k=c_k)$, or equivalently, $r_{jk}=f_R(b_k,c_k,\epsilon_R)$.
\item Show $b_k$ to a trusting voter $i$, presenting it as a sample
  from $P(D_k|T_k=t_k)$, and allow user $i$ to pick $y_i$ according to the
  trusting voter model.
\item Pick $u_{i} \sim
  P(U_{j}|R_{jk}=r_{jk},S_{jk}=s_{jk},Y_{ik}=y_{ik})$, or
  equivalently, let $u_{j}=f_U(r_{jk},s_{jk},y_{jk},\epsilon_U)$.
\end{itemize}
To distinguish the two utility functions, we will henceforth use
$f_U\tv$ for the utility function $f_U(s,y)$ used in the trusting voter model,
and use $f_U\bp$ for the utility function $f_U(r,s,y)$ defined above.
As below, we will assume pundit $j$ will pick $b_{k}$, based on available estimates of
  $S_{ij}$ and knowledge of $C_k$, to maximize the expected utility $u_j$. This can be computed as
\[ u_j = \sum_{r,s,y} P\tv(Y_{ik}=y|D_k=b_k)
      P(S_{jk}=s|D_k=c_k) P(R_{jk}=r|b_k=c_k) f_U\bp(r,s,y,\epsilon_U)
\]
where $P\tv(Y_{ik}|D_k=b_k)$ is estimated using the trusting voter
model; $P(S_{jk}|D_k=c_k)$ is computed as in
Section~\ref{sec:trusting-inference}, using $P(DC_k|T_k)$ and
$P(S_{jk}|T_j,T_k)$; and $P(R_{jk}=r|b_k,c_k)$ is computed using the
given probability function of reputational cost $r$, as a function of
the unbiased data $c_k$ and the biased version $b_k$ that is released.

Since pundit $j$ picks $B_k$ to maximize utility, then as before, we
can convert this MAID to a Bayes net.  Specifically, $B_k$ depends on
the parents $C_k$ and $S_{jk}$ as follows.
\begin{eqnarray*}
\lefteqn{P\bp(B_k=b|C_k=c,S_{jk}=s) = } & &\\
    & & P\left( b=\argmax_{b'} \sum_{r,y} P\tv(Y_{ik}=y|D_k=b_k) P(R_{jk}=r|b',c_k) 
	  \int_{\epsilon_U} f_U\bp(r,s,y,\epsilon_U) d\epsilon_U \right)
\end{eqnarray*}
This can be simplified, if we assume that $f_U\bp$ and $f_R$ are deterministic:
\begin{equation} \label{eq:biased-pundit-model}
P\bp(B_k=b|C_k=c,S_{jk}=s) = 
    P\left( b=\argmax_{b'} \sum_{y} P\tv(Y_{ik}=y|D_k=b_k) f_U\bp(f_R(b',c),s,y)\right)
\end{equation}

\subsection{Suspicious voters}

\begin{figure}
\centerline{\includegraphics[width=0.7\textwidth]{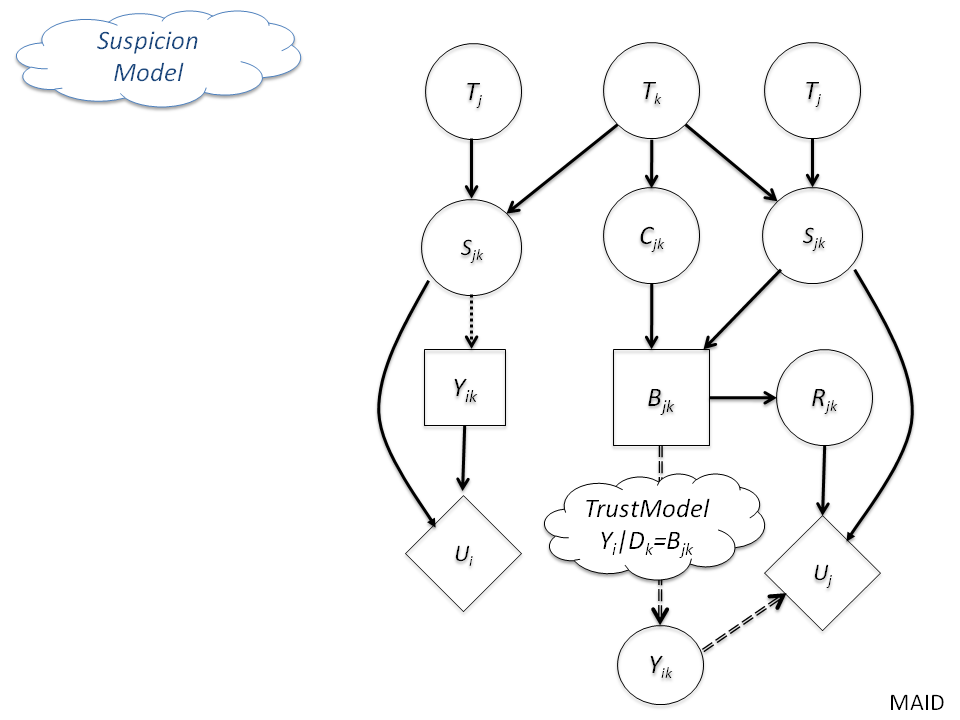}}
\caption{Model of a suspicious voter in MAID notation}
\label{fig:suspicious-voter-model-maid}

\bigskip

\centerline{\includegraphics[width=0.7\textwidth]{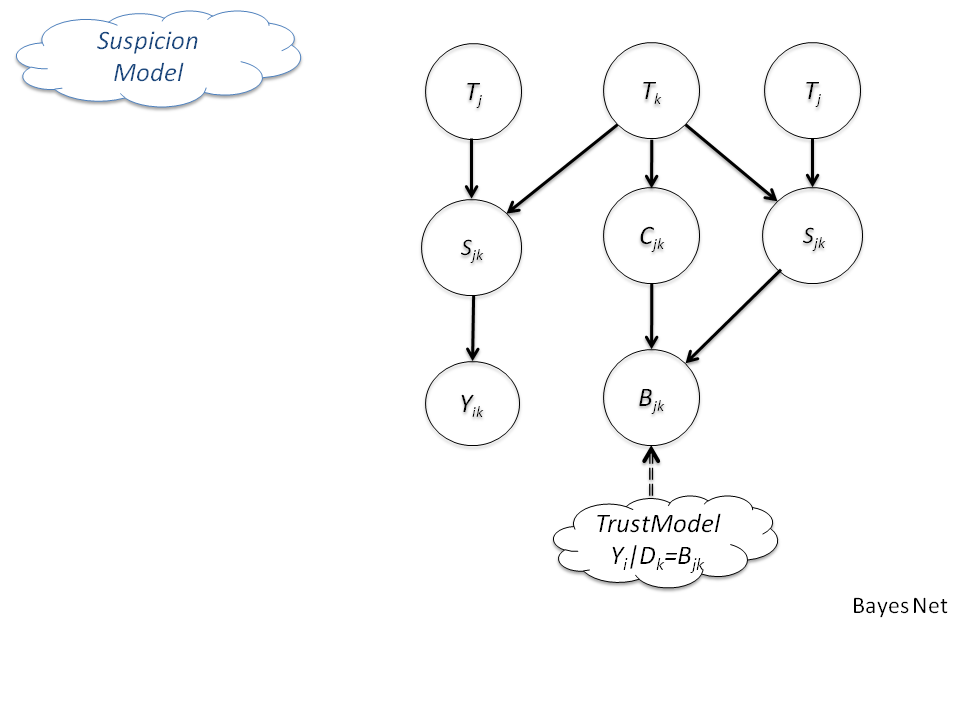}}
\caption{Model of a suspicious voter in Bayes net notation}
\label{fig:suspicious-voter-model-bn}
\end{figure}

Finally, we introduce a model for a ``suspicious voter''.
Intuitively, this model is simple. As before, we assume that $j$
chooses $b_k$ according to the biased pundit model of
Eq.~\ref{eq:biased-pundit-model}---i.e., that $j$ believes $i$ to be a
trusting voter.  The suspicious voter will then attempt to reason with
this correctly, and find the vote $Y_i$ maximizing $i$'s true utility,
given then information revealed by $j$'s choice of $b_k$.  The MAID
and Bayes net versions of this model are shown in
Figure~\ref{fig:suspicious-voter-model-maid} and
Figure~\ref{fig:suspicious-voter-model-bn} respectively, and
probabilities computed in this model will be writted as $P\sv$.

This suspicious voter model does not have a simple closed-form
solution for $P\sv(S)$, as in the trusting voter model.  The
generative process for the suspicious voter model is identical to the
biased pundit model, except that after $b_k$ is chosen according to
Eq.~\ref{eq:biased-pundit-model}, $i$ will pick a value of $Y_{ik}$
from a distribution $P\sv(Y_{ik}|B_k=b_k)$, which is chosen to
maximize the expected value of $u_i$.  We will define
\[
 P\sv(Y_{ik}=y|S_{ik}=s) = P\left(y=\argmax_{y'} \int_{\epsilon_U} f_u(y',s,\epsilon_U) d\epsilon_U\right)
\]
or, assuming determinacy,
\[
y_{ik} = f_y\sv(s_{ik}) = \argmax_{y'} f_u(y',s_{ik})
\]
This leads to a more complex inference problem for voters.  Although
the process of computing $P(S_{ik}|T_i,T_k)$ is unchanged, relative to
the trusting voter model, a suspicious voter cannot estimate a
distribution over $T_{k}$ using $d_{k}$, as in Eq.~\ref{eq:tk-from-d},
because $d_{k}$ is not known.  Instead $i$ only has indirect evidence
about $d_k$ in the form of $b_k$.

However, voter $i$ can use this indirect evidence to compute
\begin{equation} \label{eq:tk-from-bk}
P\sv(T_k|B_k=b_k) = \sum_c P\tv(T_k|D_k=c) \cdot P\bp(B_k=b|C_k=c) \cdot P\tv(D_k=c)
\end{equation}
where $P\bp(b|c)=\sum_{t_j} P\bp(b|c,t_j)P(t_j)$ is the probability, in the
biased pundit model, of pundit $j$ publishing $B_k=b$ when $C_k=c$.  We
can now break the summation over $c$ into two cases:
\begin{eqnarray} 
\lefteqn{P\sv(T_k|B_k=b_k)} & & \nonumber \\
  & = & P\tv(T_k|D_k=b) \cdot P\tv(D_k=b) \cdot P\bp(B_k=b|C_k=b) \label{eq:accurate}\\
  & + & \sum_{c\not=b} P\tv(T_k|D_k=c) \cdot P\tv(D_k=c) \cdot P\bp(B_k=b|C_k=c) \label{eq:deceptive}
\end{eqnarray}
In the term in line~\ref{eq:accurate}, $j$ is not altering the
original input $c_k$, so we say he is being \textit{accurate}.  In
line~\ref{eq:deceptive}, so we say that $j$ is being \textit{deceptive}.
In order to exploit the information in $b_k$ using Eq~
\ref{eq:accurate}-\ref{eq:deceptive} to optimize her utility, the
suspicious voter must assess the probability of deception, and adjust
inferences about candidate $k$ accordingly---a potentially difficult
inference problem.

\subsection{Implications of the suspicious voter model}

To simplify further analysis, we will make some additional
assumptions.
\begin{itemize}
\item We assume a deterministic reputational cost function $f_R$ in
  the biased pundit model.  

\item We assume the reputational cost of being accurate is zero, i.e.,
that \( \forall c, f_R(c,c)=0 \), and that the reputational cost of
being deceptive is greater than zero, i.e., \( \forall b\not=c, f_R(b,c)>0 \).

\item We assume deterministic utility functions $f_U\tv$  and $f_U\bp$.

\item We assume the utility for pundits is the same
as the utility for voters, minus the reputational cost of altering $c$
to $b$---i.e., that
\[ f_U\bp(r,s,y) \equiv f_U\tv(s,y) - r
\]
\end{itemize}

Given these assumptions, some observations can now be made.

\begin{prop}
If the prior $P(T_j)$ such that $P\bp(b|c)=P\bp(b'|c')$ for all
communications $b$, $b'$, $c$ and $c'$, then for all $b_k$,
$P\sv(T_k|b_k)=P\sv(T_k)=P\tv(T_k)$.
\end{prop}

In other words, if $P\bp(b|c)$ is constant, then a biased pundit's
publications $b_k$ convey no information to $i$.  This can be seen
immediately by inspection of Eq.~\ref{eq:tk-from-bk}.  Notice that
requiring that $P\bp(b|c)$ does not imply that any individual pundits
simply publish information $b$ uniformly at random, without regard to
$c$---instead, it says that if one averages over all pundits and
considers $\sum_{t_j} P\bp(b|c,t_j)P(t_j)$, then the cumulative
probability of seeing any particular $b$ is constant, and independent
of $c$.

More generally, one can make this observation.

\begin{prop}
If the prior $P(T_j)$ such that (1) $P\bp(c|c)=\alpha$ for all $c$,
and (2) $P\bp(b|c)=P\bp(b'|c)$ for all communications $b\not=c$ and $b'\not=c'$,
then for all $b_k$, 
\[ P\sv(T_k|b_k) = \alpha P\tv(T_k|b_k) + (1-\alpha) P\tv(T_k)
\] 
\end{prop}

In other words, if all deceptions are equally likely, but publications
are accurate with fixed probability $\alpha$, then a suspicious
voter's update to $T_k$ is simply a mixture of her prior belief
$P\sv(T_k)=P\tv(T_k)$ and the belief a trusting voter would have,
$P\tv(T_k|b_k)$, with mixing coefficient $\alpha$.  Again, this
proposition can be be verified immediately by inspection of
lines~\ref{eq:accurate}-\ref{eq:deceptive}.

A final observation is that when a biased pundit's preference $t_j$ is
the same as a voter's preference $t_i$, and this is known to both $i$
and $j$, then even a suspicious voter will obtain high utility by
simply believing $j$'s publication $b_k$.  In particular $i$'s utility
from adopting the belief that $D=b_k$ is just as high as if $i$ had
observed $c_k$ itself.

\begin{prop} \label{prop:same-alignment-utility}
If $t_j=t_i$, and $b_k$ is a publication from $j$ under the biased
pundit model, then the expected utility to $i$ of voting according to
$P\tv(T_k|D_k=b_k)$ is at least as large as the expected utility to
$i$ of voting according to $P\tv(T_k|D_k=c_k)$.  
\end{prop}

This proposition seems plausible if we recall that $b_k$ was chosen to
maximize the utility to $j$ of $i$'s belief in $b_k$ in the trusting
voter model---thus, since the utility to $i$ is the same as the
utility to $j$, it seems reasonable that adopting this belief is also
useful to $i$.  To establish it more formally, let us define 
$EU_\ell(b|c,t_i)$ to be the expected utility to agent $\ell$ (either
$i$ or $j$), absent reputational costs, of having $i$ adopt the belief
in the trusting-voter model that $D_k=b$ when in fact $D_k=c$, if
$T_i=t_i$.  In other words, we define
\[ EU_\ell(b|c,t_i) \equiv \sum_{s_{\ell k},t_k,y_{ik}} P\tv(T_k=t_k|D_k=c) P(S_{\ell k}=s_{\ell k}|t_\ell,t_k) P\tv(Y_{ik}=y_{ik}|T_i=t_i,D_k=b) f\tv_u(s_{\ell k},y)
\]
Note that the weighting in the factors $P\tv(T_k=t_k|D_k=c)
P\tv(S_{\ell k}=s_{\ell k}|t_\ell,t_k)$ holds for both $i$ and $j$,
because here we care about the true distribution over $T_k$, as
deduced from $c$.  The weighting in the factor
$P\tv(Y_{ik}=y|T_i=t_i,D_k=b)$ arises because for both $i$ and $j$,
utility is based on $i$'s estimated support $y_{ik}$ for $k$ given
$i$'s known preference $t_i$.  

If we assume that $t_i=t_j=t$, then this simplifies to
\[ EU_i(b|c,t_i) = EU_j(b|c,t_i) = \sum_{s,t_k,y_{ik}} P\tv(T_k=t_k|D_k=c) P\tv(S=s|t,t_k) P\tv(Y_{ik}=y_{ik}|T_i=t,D_k=b) f_u(s,y)
\]
and hence we see that in this case, the functions for $i$ and $j$ are
indeed the same.  Since $j$ has chosen $b$ to maximize
$EU_\ell(b|c)-f_R(b,c)$, and $f_R$ is never negative, clearly \(
EU_j(b|c,t_i) \geq EU_j(b|c,t_i) \), and so \( EU_i(b|c,t_i) \geq EU_i(b|c,t_i) \)
as well.

As noted in Section~\ref{sec:related}, there are a number of papers
analyzing media bias in which voters are assumed prefer ``good news''
(i.e., new biased towards their favored candidates) leading to
fragmentation and specialization as media companies differentiate by
providing news biased for their readers.  The results above suggest a
rational reason for picking a news source $j$ with the same partisan
preferences as one's self: in particular, this sort of news is
computionally easier to process. Similarly, biased sources with
unknown preferences are ``less informative'', in the sense that new
information leads to changes in support (relative to unbiased sources,
or well-aligned partisan sources).

\subsection{Suspicious voters and ``irrationality''}

Finally, we address the question of whether suspicous voters can
behave in the counter-intuitive manner discussed in the
introduction---whether information about the candidate $k$ that is
negative (as interpreted by a trusting voter) can increase support for
a suspicous voter.  We will show that this is possible.

\begin{theorem} \label{thm:motivated}
It can be the case that information $b$ will decrease $i$'s support
for $k$ in the trusted voter model, and increase $i$'s support for $k$
in the suspicous voter model: i.e., it may be that
$E_{P\tv}[Y_{ik}|D_k=b] < E_{P\tv}[Y_{ik}]$ but 
$E_{P\sv}[Y_{ik}|D_k=b] > E_{P\sv}[Y_{ik}]$.
\end{theorem}

Intuitively, this happens when $i$ believes strongly that $j$ is being
deceptive, and $i$ has different candidate preferences from $j$.  The
theorem asserts the existence of such behavior, so we are at liberty
to make additional assumptions in the proof (preferably, ones that
could be imagined to hold in reality).

In the proof, we assume that $j$'s preferences $T_j$ are known to $i$.
This is plausible since context may indicate, for instance, that $j$
is a strong conservative.  This does not affect the basic model, since
we allow the case of an arbitrary prior on $T_j$.

We also assume that all information about candidates is either
strictly positive for $i$ and strictly negative for $j$, or else
strictly negative for $i$ and strictly positive for $j$.  To see how
this is possible, first imagine a hypercube-like space of candidate
positions, as in Figure~\ref{fig:neg-info}, and assume that $T_i$ and
$T_j$ are on opposite corners of the cube.  The cube might indicate,
for example, positions on the environment, abortion, and increased
military spending, with $i$ preferring the liberal position on all
three and $j$ preferring the conservative positions.  Then assume that
all information indicates the probability of the candidate's position
along these three axis; in this case the assumption is satisfied.

Given these two assumptions, the statement of the theorem holds.
First, we need a slightly stronger version of
Theorem~\ref{thm:strictly-neg}.  Informally, this states that if $i$
has partial knowledge of $b$, but does know that $b$ is strictly
negative for $i$, then $i$'s support for $k$ will be weakened.

\begin{corr} \label{thm:strictly-neg-mixtures}
Suppose $H$ is a probability distribution over items of information
$b$, and also $b$ is strictly negative for $i$ for every $b$ with
non-zero probability in $H$.  Let $P\tv(Y_{ik}|H)$ denote 
\[ \sum_{b} P\tv(Y_{ik}|D_k=b)P_H(b).
\]
Then $E_{P\tv}[Y_{ik}|H] < E_{P\tv}[Y_{ik}]$.
\end{corr}

\begin{proof}
For any item $s\in dom(S_{ik})$, it is clear that
\[ P(Y|E) = \sum_{b} P(Y|b) P_H(b)
\]
and also, by marginalization over the (unrelated) variable $H$, we see
that $P(Y) = \sum_{b} P(Y) P_H(b)$.  Since for all $b$ with non-zero
probability under $H$ we have that $E[Y|b]<E[Y]$, the result holds.
\end{proof}

We can now prove Theorem~\ref{thm:motivated}.

\begin{proof}
Suppose information $b$ that is strictly negative for $i$ is observed
by $i$, and consider again the formula for $P\sv(T_k|B_k=b_k)$ given
on lines~\ref{eq:accurate} and \ref{eq:deceptive}.  This shows that
the suspicous voter will reason by cases.  In one case, $j$ is being
accurate, and the change in probability for $T_k$ is in the same
direction as in the trusted voter model, and as noted above, this will
lead to a belief update that decreases support for $k$.  However, this
change is down-weighted by the factor $P\bp(B_k=b|C_k=b) P\tv(D_k=b)$,
which can be interpreted as the probability that $b$ was really
observed times the probability that $j$ chooses to report accurately.
We will assume that $P\tv(D_k=b)$ is very small, so that
\[ P\sv(T_k|B_k=b_k) \approx \sum_{c\not=b} P\tv(T_k|D_k=c) \cdot P\tv(D_k=c) \cdot P\bp(B_k=b|C_k=c)
\]
In this case, $j$ is being deceptive.  

Reasoning is this case is complicated by the fact that $i$ must
consider all inputs $c$ that could have observed, and compute the
product of $P\tv(D_k=c)$, the prior probability of $c$, and also
$P\bp(B_k=b|C_k=c)$, the probability of $b$ being reported in place of
$c$ by $j$.  However, since the preferences of $i$ and $j$ are
opposite, the latter is quite informative: in particular, since $b$
was deceptively chosen by $j$ to be negative for $i$, then $b$ must
prefer that $i$ give weaker support for $k$, implying that $c$ is
actually negative for $j$, and hence positive for $i$.  This holds for
every $c$ that could have let to the deceptive report $b$.  By
Corollary~\ref{thm:strictly-neg-mixtures}, the net change in support
for $i$ in the suspicous voter model positive.
\end{proof}

\section{Concluding Remarks}

To summarize, we propose a model in which there are two classes of
voters, \trm{trusting voters} and \trm{suspicious voters}, and two
types of information sources, \trm{unbiased sources} and \trm{biased
  sources}.  Information from an unbiased source is modeled simply as
observations $D_k$ that probabilistically inform a voter about a
candidate $k$'s positions, and \trm{trusting voters} are voters that
treat information about a candidate as coming from an unbiased source.
We show that reasoning about new information is computationally easy
for trusting voters, and that trusting voters behave intuitively: in
particular, negative information about candidate $k$ (according to a
particular definition) will decrease support for $k$, and positive
information will increase support.

\trm{Biased sources} are information sources $j$ who plan their
communications in order to encourage trusting voters to vote in a
particular way).  To do this, they report some possibly-modified
version $B_{k}$ of a private observation $C_{k}$, concealing the
original $C_{k}$.  In the model, $B_{k}$ is chosen based on the
utility to $j$ of the probable effect of $B_{j}$ on a trusting voter
$i$.

Finally, suspicious voters model the behavior of biased sources.  In
general this is complicated to do, however, some special cases lead to
simple inference algorithms.  For instance, under one set of
assumptions, all information from biased sources can be ignored.  In
another set of assumptions, a suspicious voter will make the same sort
of updates to her beliefs as a trusting voter, but simply make them
less aggressively, discounting the information by a factor related to
the probability of deception.  

Another interesting tractible reasoning case for suspicious voters is
when the biased information source $j$ has the same latent candidate
preferences as the voter $i$.  In this case, suspicious voters can act
the same way a trusting voter would---even if the information acted on
is false, it is intended to achieve a result that is desirable to $i$
(as well as $j$).

These results are of some interest in light of the frequently-observed
preference for partisan voters to collect information from similarly
partisan information channels.  The results suggest a possibly
explanation for this, in terms of information content.  The optimal
way to process information from a possibly-deceptive unknown source,
or a source with a known-to-be-different partisan alignment, is to
either discount it, ignore it, or else employ complicated (and likely
computationally complex) reasoning schemes.  However, reports from
partisan source with the same preferences as a voter can be acted on
as if they were trusted---even if the reports are actually deceptive.

Finally, we show rigorously that a suspicious voter can, in some
circumstances, increase support for a candidate $k$ after receiving
negative information about $k$.  Specifically, information that would
decrease support for a trusting voter might increase support for a
suspicious voter---if she believes the source has different candidate
preferences, and if she believes the source is being deceptive.  This
behavior mimics behavior attributed elsewhere to ``motivated
reasoning'', but does not arise from ``irrationality''---instead it is
a result of the voters correct identification of, and compensation
for, an ineffective attempt at manipulation on the half the
information source.

We should note that this hypothesis does \emph{not} suggest that
emotion is not present in such situations---in fact, it seems likely
that reports viewed as deceptive would indeed provoke strong emotional
responses.  It does suggest that some of the emotion associated with
these counterintuitive updates may be associated with mechanisms that
have an evolutionary social purpose, rather than being a result of
some imperfect adaption of humans to modern life.

It seems plausible that additional effects can be predicted from the
suspicious-voter model.  For instance, although we have not made this
conjecture rigorous, if a biased source $j$ does not know $i$'s
political preferences, then deceptive messages $b$ will likely tend to
be messages that would be interpreted as negative by most voters.  For
instance, $j$ might assert that the candidate violates some cultural
norm, or holds an extremely unpopular political views. (Or , on the
other hand, assert the candidate has a property that almost all voters
agree is ``good''.  Arguably, most information from deceptive partisan
sources would be of this sort, rather than discussion of stands on
widely-disagreed-on issues (like gay marriage or abortion in the US).

As they stand, however, the results do suggest a number of specific
predictions about how information might be processed in a social
setting. First, 
information provided by persons believed to have political
  alignments similar to voter $i$ will be more easily assimilated, and
  have more effect on the view of $i$, than information provided by
  persons believed to have different political alignments.
Second, information provided by persons with political alignments
  similar to voter $i$ will be more assimilated in roughly the same
  speed (and with the same impact) as information from a
  believed-to-be-neutral source.
Third, information provided by persons with political alignments
  different from voter $i$ may lead to counterintuitive updates, while
  information from similarly-aligned sources or neutral sources will
  not.
An important topic for future work would be testing these predictions, for instance using
the DPTE methodology.

\bibliographystyle{plain}
\bibliography{politics}

\end{document}